\documentclass[final, 5p, times, twocolumn]{elsarticle}

\biboptions{semicolon, sort&compress}

\usepackage{amsfonts}
\usepackage{subfigure}

\usepackage{amsmath}
\usepackage{amssymb,amsthm}

\usepackage{url}

\usepackage{threeparttable}
\usepackage{array}
\usepackage{multirow}
\usepackage{diagbox}
\newcommand{\PreserveBackslash}[1]{\let\temp=\\#1\let\\=\temp}
\newcolumntype{C}[1]{>{\PreserveBackslash\centering}p{#1}}
\newcolumntype{R}[1]{>{\PreserveBackslash\raggedleft}p{#1}}
\newcolumntype{L}[1]{>{\PreserveBackslash\raggedright}p{#1}}

\usepackage{pgfplots}
\pgfplotsset{compat=1.13}

\newtheorem{theorem}{Theorem}
\newtheorem{lemma}{Lemma}
\theoremstyle{definition}

\newcommand{\trans}{\mathcal{T}}
\newcommand{\rot}{\mathcal{R}}
\newcommand{\perm}{\mathcal{P}}
\newcommand{\diagperm}{\mathcal{D}}

\begin{document}

\begin{frontmatter}
\title{Deep Rotation Equivariant Network}

\author[cad]{Junying~Li}
\ead{microljy@zju.edu.cn}

\author[ccnt]{Zichen~Yang}
\ead{zichenyang.math@gmail.com}

\author[ccnt]{Haifeng~Liu}
\ead{haifengliu@zju.edu.cn}

\author[cad]{Deng~Cai\corref{cor}}
\ead{dengcai@gmail.com}

\cortext[cor]{Corresponding author}

\address[cad]{The State Key Laboratory of CAD\&CG, College of Computer Science, Zhejiang University, China}
\address[ccnt]{College of Computer Science, Zhejiang University, China}

\begin{abstract}
Recently, learning equivariant representations has attracted considerable research attention. Dieleman et al. introduce four operations which can be inserted into convolutional neural network to learn deep representations equivariant to rotation. However, feature maps should be copied and rotated four times in each layer in their approach, which causes much running time and memory overhead. In order to address this problem, we propose Deep Rotation Equivariant Network consisting of cycle layers, isotonic layers and decycle layers.
Our proposed layers apply rotation transformation on filters rather than feature maps, achieving a speed up of more than 2 times with even less memory overhead. We evaluate DRENs on Rotated MNIST and CIFAR-10 datasets and demonstrate that it can improve the performance of state-of-the-art architectures.
\end{abstract}


\end{frontmatter}

\section{Introduction}
\label{sec::intro}
Convolutional neural networks(CNNs) recently have made great success in computer vision tasks\cite{cnn1,cnn2,cnn3,cnn4,he2016deep}. One of the reasons to its success is that weight sharing of convolution layers ensures the learnt representations are translation equivariant\cite{group}, i.e., shifting an image and then feeding it through the network is the same as feeding the original image and then shifting the resulting representations.

However, CNNs fail to exploit rotation equivariance to tackle vision problems on datasets with rotation symmetry in nature, especially microscopic images or aerial images, which can be photographed from any angle. Thus, current studies focus on dealing with this issue.

One widely used method to achieve rotation equivariance is to constrain the filters of the first convolutional layer to be rotated copies of each other, and then apply cross-channel pooling immediately after the first layer\cite{wu2015flip,marcos2016learning,teney2016learning}. However, only shallow representations equivariant to rotation can be learnt by applying one convolutional layer. In addition, such representations are nearly trivial, since pooling rotated copies is approximately equivalent to convolving non-rotated inputs with highly symmetric filters. 

To solve this problem, \cite{dieleman2016exploiting} introduces four operations which can be combined to make these models able to learn deep representations equivariant to rotation. However, every feature map should be copied and rotated by these operations four times, which causes high memory and running time overhead. 

In this paper, we give a comprehensive theoretical study on approaches to rotation equivariance with CNNs. We propose a novel CNN framework, Deep Rotation Equivariant Network(DREN) to obtain deep equivariance representations. We prove that DREN can achieve the identical output to that of \cite{dieleman2016exploiting} with much less running time and memory requirements. 

We evaluate our framework on two datasets, Rotated MNIST and CIFAR-10. On Rotated MNIST, it can outperform the existing methods with less number of parameters. On CIFAR-10, it can improve the results of state-of-the-art models with the same number of parameters. Moreover, our implementation achieves a speed up of more than 2 times as that of \cite{dieleman2016exploiting}, with even less memory overhead. 

\section{Related Works}
Learning invariant representations by neural networks has been studied for over a decade. Early works focus on refinement of restricted Boltzmann machines(RBMs) and deep belief nets(DBNs). \cite{kavukcuoglu2009learning} gives an approach to automatically generate topographic maps of similar filters in an unsupervised manner and these filters can produce local invariance when being pooled together. \cite{norouzi2009stacks} develops convolutional RBM(c-RBM), using weight sharing to achieve shift-invariance. Later, a following work by \cite{schmidt2012learning} incorporates linear transformation invariance into c-RBMs, yielding features that have a notion of transformation performed. The model proposed by \cite{lee2009convolutional} uses a probabilistic max-pooling layer to support efficient probabilistic inference, which also shows the property of translation invariance.

Recently, convolutional neural networks have become the most popular models in various computer vision tasks\cite{yu2017multi,yu2017iprivacy,zhang2017learning,grzegorczyk2016encouraging}. One of the advantages of CNNs is its translation equivariant property provided by weight sharing\cite{group}. However, it cannot deal with rotation transformation of input images. Thus, many variants of CNNs have been proposed to settle these problems. Basically, the idea of most of the related works\cite{fasel2006rotation,teney2016learning,dieleman2015rotation} is to stack rotated copies of images or features to obtain rotation equivariance. 

There are also other methods. \cite{gens2014deep} propose deep symmetry networks that can form feature maps over arbitrary transformation groups approximately. Methods proposed by \cite{wu2015flip, marcos2016learning} show that rotation convolution layers followed by a cross-channel pooling over rotations could achieve rotation equivariance. In fact, none of the directional features could be extracted by these methods, since pooling is applied right after one rotation convolution. \cite{group} propose a group action equivariant framework by stacking group acted convolution and provide a theoretically grounded formalism to exploit symmetries of CNNs. \cite{worrall2016harmonic} present harmonic networks, a CNN structure exhibits equivariance to patch-wise translation and 360-rotation. Recently, the vector field network proposed by \cite{marcos2016rotation} applies interpolation to deal with rotation of general degrees. 

\cite{dieleman2016exploiting} introduces four operations to encode rotation symmetry into feature maps to build a rotation equivariant neural network. However, feature maps should be rotated each time to ensure equivariance, which obviously costs much time and memory. Our approach presents a different way to overcome this issue by rotating filters, which brings about exactly the same results, but in a more efficient way. 

\section{Equivariance and invariance}
In this section, we briefly discuss the notions of equivariance and invariance of image representations. Formally, a representation of a CNN can be regarded as a function $f$ mapping from image spaces to feature spaces.

We say, a representation $f$ is equivariant to a family $\mathfrak{T}$ of transformations on image spaces, if for any transformation $\trans\in\mathfrak{T}$, there exists a corresponding transformation $\trans '$ on feature spaces, such that 
\begin{equation}
f(\trans x)=\trans 'f(x),
\end{equation}
for any input images $x$. Intuitively, this means that the learnt representation $f$ of CNNs changes in an expected way, when the input image is transformed. 

There is another stronger case when $\trans '$ is the identity map, i.e., the map fixing the inputs, for all $\trans\in\mathfrak{T}$. This indicates that the representations remain unchanged no matter how the input data is transformed by transformations in $\mathfrak{T}$, i.e., the representation is invariant. Invariance is an ideal property of representation, because a good object classifier must output an invariant class label no matter what location of the object lies in.

\begin{figure}
\label{fig::rotequiv}
\begin{center}
\includegraphics[width=0.47\textwidth]{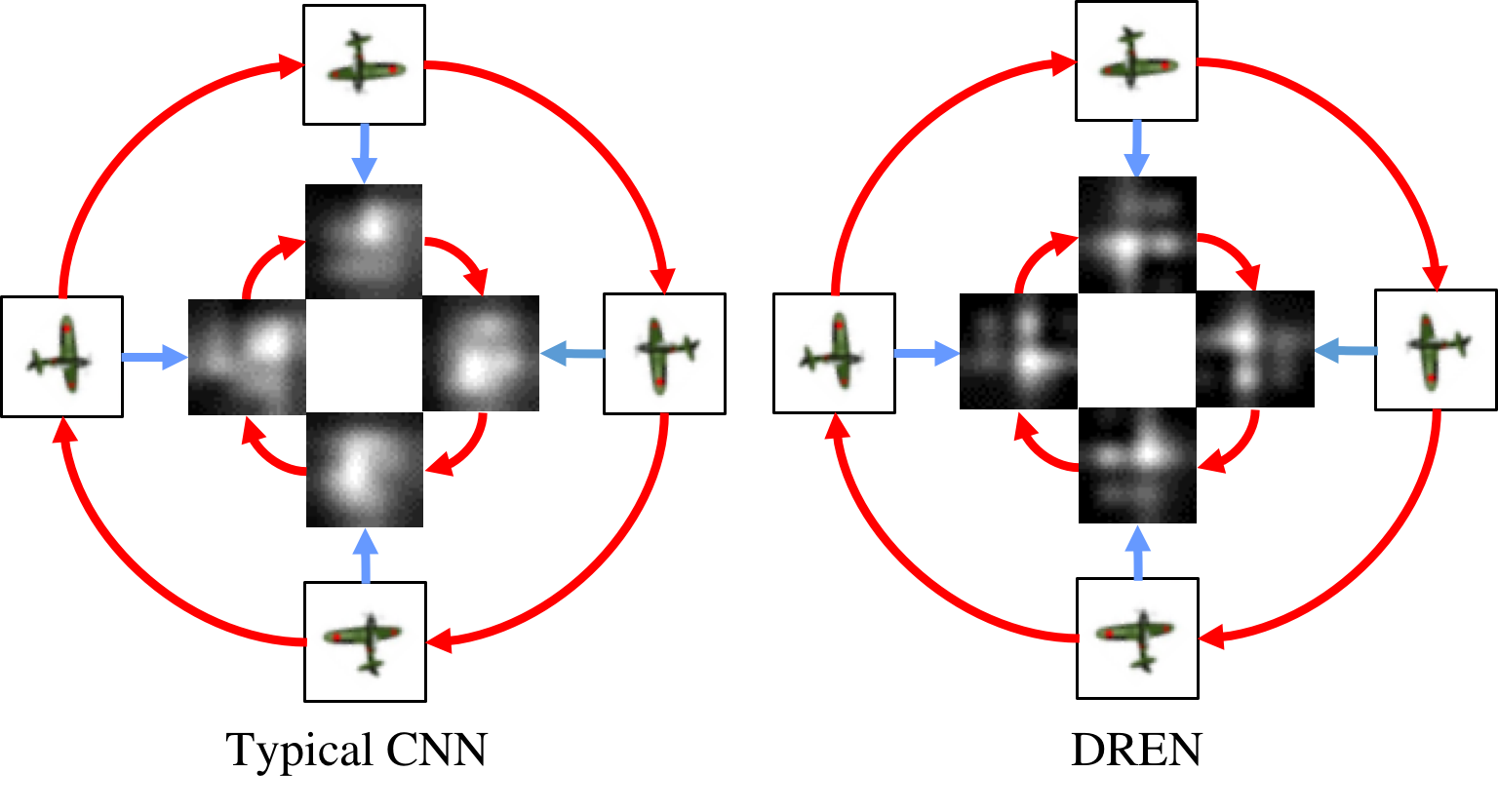}
\end{center}
\caption{\textbf{Latent representations learnt by a CNN and a DREN(Proposed)}, where $R$ stands for clockwise rotation. The left part is the result of a typical CNN while the right one is that of a DREN. In both parts, the outer cycles consist of the rotated images while the inner cycles consist of the learnt representations. Features produced by a DREN is equivariant to rotation while that produced by a typical CNN is not.}
\end{figure}

The goal of this paper is to present a novel convolutional neural network framework, which learns representations that are equivariant to rotation transformations $\rot$, s.~t. $\rot '=\rot$, that is
\begin{equation}
\label{eqn::goal}
f(\rot x)=\rot f(x),
\end{equation}
for any input image $x$. 
Figure 1 gives an example of rotation equivalent representations learnt by DREN, comparing to a traditional CNN.
The reason that we do not directly work on rotation invariant representations is that this kind of rotation equivariance can be easily lifted to rotation invariance, for instance, by a global pooling operation\cite{nin}, i.e., the kernel size of this pooling layer is equal to the size of feature maps.

Since these are the only four kinds of possible rotation of an image that can be performed without interpolations, we mainly deal with the rotation transformation family $\mathfrak{R}=\{\rot_\theta|\theta=k\pi/2, k\in\mathbb{Z}\}$. However, our experimental results show that our framework can achieve good performance when dealing with rotation for general degrees. 

\begin{figure*}
\begin{center}
\includegraphics[width=0.9\textwidth]{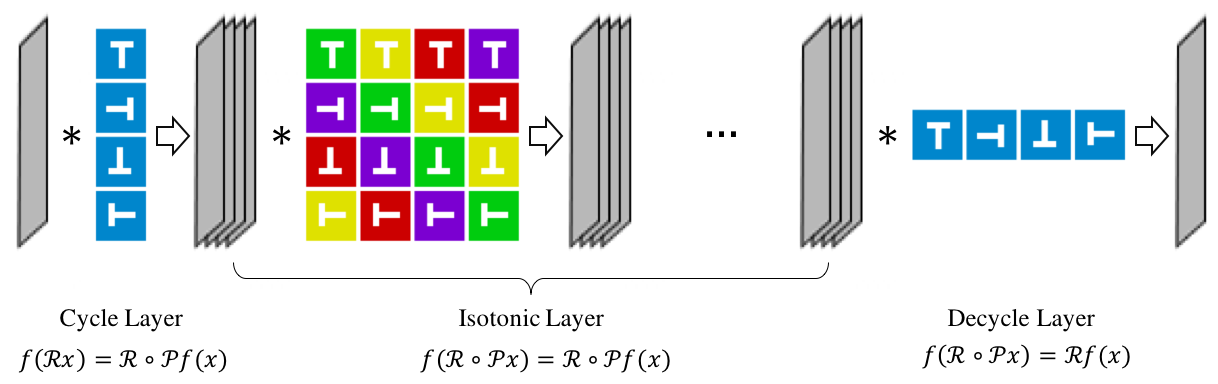}
\end{center}
\caption{\textbf{The framework of Deep Rotation Equivariant Network.} The gray panels represent input, feature maps and output. Each square represents a weight kernel. The letter 'T' is utilized for distinguishing orientation clearly. The different colors of kernel background indicate that the kernel are not qualitatively equivalent. Although this figure seems similar to that one in \cite{dieleman2016exploiting}, there are 3 critical differences: 1. We apply rotation on the filters rather than the feature maps. 2. The matrix in an isotonic layer is different in order from the matrix of cyclic rolling operation in \cite{dieleman2016exploiting}. 3. The decycle layer is a special convolution layer, different from the cyclic pooling applied in \cite{dieleman2016exploiting}.}
\label{fig:framework}
\end{figure*}

\section{Rotation equivariant convolution}
In this section, we define three novel types of convolutional layers, which are combined to learn rotation equivariant features. 

\subsection{Preliminaries}
\label{sec::Preliminaries}
For the sake of simplicity, we omit bias terms, activation functions and other structures, concentrating on convolution. In addition, we set the stride of any convolution layer be $1$. The general case will be discussed in Section \ref{sec::framework}.

First, we vectorize convolution operation formally to simplify our derivation. Shortly, we shall use matrix multiplication to describe multi-channel convolution. Let us assume that the input of a convolutional layer contains $n$ feature maps(images) $\{x_j\}_{j=1}^{n}$. This layer has $mn$ filters, denoted by $W_{ij}$ with $1\leq i\leq m, 1\leq j\leq n$. These can be organized as a matrix(vector) $x$ of size $n$ and a matrix $W$ of size $m,n$ in which entries are filters or feature maps. We refer to such a matrix(vector) a hyper-matrix(hyper-vector). Then, the convolution $W*x$ is defined to be a hyper-vector of size $m$ whose $i$-th entry is,
\begin{equation}
\sum_{1\leq j\leq n}W_{ij}*x_j,
\end{equation}
for each $1\leq i\leq m$. One can verify that this is actually equivalent to ordinary multi-channel convolution. 

Next, we introduce the rotation operator $\rot$, rotating a filter or a feature map by a degree of $\pi/2$ counterclockwise. We also define the action of $\rot$ on a hyper-matrix $W$. $\rot(W)$ is defined to be entrywise rotation. There are three obvious facts about the rotation operator that are frequently used in the sequel. 
\begin{enumerate}
\item Convolutionally distributive law: $\rot(W*x)=\rot(W)*\rot(x)$. This indicates that rotating filters and feature maps simutaneously before convolution yields rotated outputs.  
\item Additively distributive law: $\rot(X_1+X_2)=\rot(X_1)+\rot(X_2)$. Note that $+$ is entrywise addition of matrices. 
\item Cyclic law: $\rot^4$ is equal to the identity transformation.
\end{enumerate}

In our framework, we propose three new types of convolution layers to implement rotation equivariance. Given an image, a cycle layer transforms this image into $n_0$ groups, each containing $4$ feature maps. Then, one can apply $k$ isotonic layers consecutively to learn deeper representations. Suppose that the output of the $i$-th isotonic layer consists of $n_i$ groups with $1\leq i\leq k$, each containing $4$ feature maps as well. Finally, a decycle layer applies a convolution to merge feature maps, resulting in $n_{k+1}$ rotation equivariant feature maps.

Informally, the utility of the cycle layers is to transform rotation of images into cyclic permutations of feature maps in each group. Then, the isotonic layers are used to preserve the order of permutations in each group. Finally, the decycle layer decodes permutations in each group to produce rotation equivariant representations.

In the following discussion, assume that the number of groups $n_i=1$ for any $0\leq i\leq k+1$ . The general case will be discuss in Section \ref{sec::framework}

\subsection{Cycle layers}
\label{sec::flayer}
The basic idea of cycle layers is to stack rotated copies of the filters. By assumption, the hyper-matrix of the filters $W_{in}$ in a cycle layer contains $4$ filters, say $W_{in}=[w, \rot w, \rot^2w, \rot^3w]^{T}$. Given an input image $x$, by the definition of convolution, the output is
\begin{eqnarray*}
f_{in}(x)&=&W_{in}*x\\
         &=&[w*x,\rot w*x,\rot^2w*x,\rot^3w*x]^{T}.
\end{eqnarray*}
Next, we shall derive the relation between cycle layers and rotation operators. Before it, we have to define a cyclic permutation operator $\perm$ acting on a hyper-vector, adding the index of each entry by $1$. Note that if the added index is $4$, we reset it to $1$. For instance, $\perm(x_2)=x_{3}, \perm(x_4)=x_{1}$. This description is similar to the modular operation, thus we refer this operation addition modulo $4$, denoted by $x+1\ (\text{mod }4)$.

When a rotated image $\rot x$ is fed into a cycle layer, the output becomes
\begin{equation}
\begin{split}
\label{eqn::layer1}
f_{in}(\rot x)&=W_{in}*\rot x\\
&=[w*\rot x,\rot w*\rot x, \rot^2w*\rot x, \rot^3w*\rot x]^{T}\\
&=\rot[\rot^3w*x,w*x,\rot w*x,\rot^2w*x]^{T}\\
&=\rot\circ \perm(W_{in}*x)=\rot\circ \perm f_{in}(x),
\end{split}
\end{equation}
where the third equality follows from the cyclic law and the convolutional distributive law. This indicates that given a rotated image, the output $f_{in}(\rot x)$ of a cycle layer is the same as $f_{in}(x)$ up to a permutation and a rotation. Hence, the rotation of images is transformed into the order of cyclic permutation.

Although, cycle layers cannot produce rotation equivariance, one can immediately obtain it by a following cross-channel pooling operation. In this case, this operation produces a feature map by taking average or pixelwise maximum of the four rotated feature maps. But this is equivalent to use a highly symmetric filter, which generates nearly trivial features. Hence, we propose isotonic layers to replace it.

\subsection{Isotonic layers}
\label{sec::isotonic}
To produce non-trivial feature maps, a straightforward way is stacking more convolutional layers. However, one sees that an ordinary convolution operation will destroy the order of cyclic permutations produced by the cyclic layer. Consequently, we have to study that under what conditions, convolution can preserve the order of cyclic permutations. Mathematically, we can formulate this idea by the following equation,
\begin{equation}
\label{eqn::layer2}
f_{hide}(\rot\circ\perm x)=\rot\circ\perm f_{hide}(x),
\end{equation}
where $f_{hide}$ stands for the convolution performed by an isotonic layer. Intuitively, this means that feeding permuted rotated inputs, an isotonic layers yields the equivariant permuted rotated feature maps, hence preserving the order of cyclic permutation. By induction, one sees that a stacking of $h$ isotonic layers share the same property, where $h$ is any non-negative integer.

To derive equivalent conditions to Equation (\ref{eqn::layer2}), we define diagonally permutation operator $\diagperm$ acting on a two-dimensional hyper-matrix(not a hyper-vector), adding two indices of each entry by $1$ modulo $4$, for example, $\diagperm(W_{ij})=W_{i+1\ (\text{mod }4), j+1\ (\text{mod }4)}$. With this, we can show that how the cyclic permutation operator $\perm$ interacts with convolution.
\begin{lemma}
\label{thm::pwx}
Given an input hyper-vector $x$ of size $4$ and a square hyper-matrix $W$ of size $4$, we have,
$\perm(W*x)=\diagperm W*\perm x.$
\end{lemma}
\begin{proof}
For any $1\leq j\leq 4$, let $y_j$ be the $j$-th column hyper-vector of $W*x$ and $y'_{j}$ be the $j$-th column hyper-vector of $\diagperm W*\perm x$. By the definition of convolution, one has,
\begin{align*}
\perm(y_j)&=y_{j+1(\text{mod }4)}\\
      &=\sum_{i=1}^{4}W_{i,j+1(\text{mod }4)}*x_{i}\\
      &=\sum_{i=1}^{4}W_{i+1(\text{mod }4),j+1(\text{mod }4)}*x_{i+1(\text{mod }4)}\\
      &=\sum_{i=1}^{4}\diagperm W_{ij}*\perm x_i\\
      &=y'_{j}.
\end{align*}
Thus, $\perm(W*x)=\diagperm W*\perm x$;
\end{proof}

By Lemma \ref{thm::pwx}, suppose $W_{hide}$ as the weight of isotonic layer, one can show that,
\begin{align*}
        \quad f_{hide}(\rot\circ\perm x) &= \rot\circ\perm f_{hide}(x)\\
\Leftrightarrow \quad W_{hide}*\rot\circ\perm x  &= \rot\circ\perm (W_{hide}*x)\\
\Leftrightarrow \quad W_{hide}*\rot\circ\perm x  &= \rot(\diagperm W_{hide}*\perm x)\\
\Leftrightarrow \quad W_{hide}*\rot\circ\perm x  &= \diagperm\circ\rot W_{hide}*\rot\circ\perm x.
\end{align*}
Since $x$ can be an arbitrary image, one has,
\begin{equation}
\label{eqn::condition}
W_{hide}=\diagperm\circ\rot W_{hide}.
\end{equation}
By straightforward calculation, we find a special class of hyper-matrices satisfying Equation (\ref{eqn::condition}), which yields the following theorem.
\begin{theorem}
A hyper-matrix $W_{hide}$ satisfies Equation (\ref{eqn::condition}), if and only if it is of the form
$$\left(\begin{array}{cccc}
A    &    B &    C & D\\
\rot D   &   \rot A &   \rot B & \rot C\\
\rot ^2C & \rot ^2D & \rot ^2A & \rot ^2B\\
\rot ^3B & \rot ^3C & \rot ^3D & \rot ^3A
\end{array}\right),$$
where $A, B, C, D$ are square filters of the same size.
\end{theorem}

Thus, we restrict the hyper-matrix of the filters in isotonic layers to be the form in previous theorem. Therefore, an isotonic layers has the desired property, i.e., Equation (\ref{eqn::layer2}). This can be easily extended to general cases.

\subsection{Decycle layers}
After a stacking of isotonic layers, we have learnt a very deep representation, with the order of cyclic permutations. Thus, the task of decycle layers is to transform cyclic permutation information into rotation equivariant representations. Formally, we require the following property,
\begin{equation}
\label{eqn::layer3}
f_{out}(\rot\circ\perm x)=\rot f_{out}(x),
\end{equation}
where $f_{out}$ is the convolution of a decycle layer. The meaning is obvious. The cyclic permutation operator is removed, reducing to rotation equivariance. Similarly, we can derive an equivalent condition for $W_{out}$, the weight of decycle layer. This requires two fact that, if $W_{out}$ is a hyper-vector, one has $W_{out}*\perm x=\perm ^{-1}W_{out}*x$ and $\rot\circ\perm=\perm\circ\rot$. Thus, the derivation is shown below,
\begin{align*}
                f_{out}(\rot\circ\perm x) &= \rot f_{out}(x)\\
\Leftrightarrow \quad   W_{out}*(\rot\circ\perm x)  &= \rot(W_{out}*x)\\
\Leftrightarrow \quad   W_{out}*(\perm\circ\rot x)  &= \rot W_{out}*\rot x\\
\Leftrightarrow \quad~~   \perm ^{-1}W_{out}*\rot x &= \rot W_{out}*\rot x.
\end{align*}
Since $x$ can be an arbitrary image, one has 
\begin{equation}
W_{out}=\perm\circ\rot W_{out}.
\end{equation} 
This is equivalent to say that $W$ is of the form $[w, \rot w, \rot ^2w, \rot ^3w]$. Although it looks similar to that one in cycle layers, one see that they actually differ by a transpose operation. 

Thus, a decycle layer applies filters of this form, decoding permutation information into our desired rotation equivariance, i.e., $f_{out}(\rot\circ\perm x)=\rot f_{out}(x)$.


\subsection{Architecture of DREN}
\label{sec::framework}
In this subsection, we make a careful discussion on the architecture and rotation equivariance of DREN, in detail.

\subsubsection{Framework} In the architecture of DREN, the first layer must be a cycle layer. It is followed by $k\geq0$ consecutive isotonic layers and one more decycle layer. Then, we claim that the output representation of the decycle layer is rotation equivariant, which can be easily proved by using the Equation (\ref{eqn::layer1}), (\ref{eqn::layer2}) and (\ref{eqn::layer3}) repeatly. For notational simplicity, we reuse $f_{hide}$ for different isotonic layers. Indeed, one has,
\begin{equation}
\begin{split}
f_{out}f^k_{hide}f_{in}(\rot x)&=f_{out}f^k_{hide}(\rot\circ\perm f_{in}(x))\\
&=f_{out}(\rot\circ\perm f^k_{hide}f_{in}(x))\\
&=\rot f_{out}f^k_{hide}f_{in}(x),
\end{split}
\end{equation}
where if we view $f_{out}f^k_{hide}f_{in}$ the representation $f$, the equality of two sides is exactly our desired property, i.e., Equation (\ref{eqn::goal}). Note that, at the end of Section \ref{sec::Preliminaries}, we have assumed that $n_i=1$ for $0\leq i\leq k+1$. In fact, this property can be proved for the general case similarly.

\subsubsection{Compatibility} Moreover, adding bias term is allowable, if any channel in a group share the shared bias value. So do batch normalization layers with the shared scale and bias parameters. Since ReLU activation functions are channel indenpendent operation, they do not have any impact on rotation equivariance, thus compatible with our architecture. So do dropout layers. Moreover, one can also apply fully connected layers following the decycle layer.

In addition, we should discuss convolution and pooling layers with stride $>1$. When stride is more than $1$, one can observe that the filter may not convolve marginal pixels in an image, thus destroying rotation equivariance. Indeed, suppose the size of the image is $n$ by $n$, rotation equivariance can be preserved iff. $n=k\cdot stride+kernel_size$ for some $k\in\mathbb{N}$. The proof is trivial.

\subsubsection{Invariance} Furthermore, the equivariance property can be lifted to invariant property in our framework. Simply applying a global pooling layer\cite{nin} in the end of our framework can make output rotation invariant.


\subsection{Memory consumption analysis}
Some previous work\cite{dieleman2015rotation,dieleman2016exploiting,laptev2016ti} rotate feature maps to obtain rotation equivariance. Indeed, one can show that rotating feature maps correspondingly can yield equivalent results. However, rotating and buffering filters are more efficient, because feature maps have much larger size than filters at most of the time.

The theoretical analysis of memory usage is based on the fact that most of popular deep learning frameworks use GEMM algorithm\cite{gemm} to compute image convolution. Precisely, feature maps should be copied $k^2$ times, merged into a single matrix, then multiplied by weight matrix. Therefore, the total size of feature maps is clearly much larger than the size of filters at most of the time. This indicates that rotating the feature maps lead to about 4 times more memory cost than rotating the filters. The detail memory costs are listed in Table \ref{tbl::memory}. 

On the other side, rotating feature maps also costs more time in rotation operation and copy operation. In Section \ref{sec::exp}, we give a detailed comparison on our approach and an architecture using the strategy of rotating feature maps, proposed by \cite{dieleman2016exploiting}.

\begin{table}
\centering

\begin{threeparttable}
\caption{Memory cost of rotating filters and feature maps}
\label{tbl::memory}
\begin{tabular}{|L{3.5cm}|C{1.5cm}C{2.3cm}|}
\cline{1-3}
\diagbox{Mem. cost}{Method} & Rotate filters  & Rotate feature map \\
\cline{1-3}
Filters          & $4c_{in}c_{out}k^2$  & $c_{in}c_{out}k^2$ \\
Feature map      & $nc_{in}wh$          & $4nc_{in}wh$       \\
Feature map(GPU) & $nc_{in}whk^2$       & $4nc_{in}whk^2$    \\
\cline{1-3}
\end{tabular}
\begin{tablenotes}
\footnotesize
\item [1] Here $n,c_{in},c_{out},w,h,k$ denote batch size, input channels, output channels, width of input, height of input and kernel size, respectively.
\end{tablenotes}
\end{threeparttable}
\end{table}

\subsection{Relations to other methods}
\label{sec::relation}
Since DREN is a general framework, it can be shown that several existing networks become special cases of our framework.

For instance, if we set all element of the hyper-matrix of a decycle layer to be the identity matrix times a normalizing factor, the decycle layer becomes cross-channel mean pooling layer that is used by most of work\cite{group, wu2015flip, marcos2016learning}. 

For isotonic layers, we can set $W_{hide}$ to a diagonal form, i.e., all of filters other than those on the diagonal of hyper-matrix to be zero matrices. In other words, this means that all of feature maps would be convolved by rotated copies of a single filter, which is utilized by \cite{fasel2006rotation, dieleman2016exploiting,laptev2016ti}. 

Moreover, the last decycle layer can be replaced by a cross-channel max pooling layer. Although it is not a decycle layer in nature, the cross-channel max pooling layer satisfies Equation (\ref{eqn::layer3}) as well. Actually other layers that satisfy Equation (\ref{eqn::layer3}), such as cross-channel root-mean-square(RMS) layer, can also be applied to replace the decycle layer and to preserve rotation equivariance as well.


Finally, we claim that our approach yields the equivalent result to that of \cite{dieleman2016exploiting} when the proposed layers are carefully used. Note that, by convolutional distributive law, for any filter $W$ and any feature map $x$, $\rot(W*x)=\rot(W)*\rot(x)$. 
Then we can derive that
the cyclic slice layer\cite{dieleman2016exploiting} followed by an ordinary convolutional layer equals to the cycle layer;
the cyclic roll layer\cite{dieleman2016exploiting} followed by an ordinary convolutional layer equals to the isotonic layer; 
the cyclic pooling\cite{dieleman2016exploiting} is actually cross-channel pooling which is a special case of the decycle layer;
the cyclic stack layer\cite{dieleman2016exploiting} followed by an ordinary convolutional layer equals to the isotonic layer whose weight hyper-matrix $W_{hide}$ is a diagonal matrix. In fact, DREN operates on filters rather than feature maps, reducing about half of running time with less memory overhead compared with that of \cite{dieleman2016exploiting}

\section{Experiments}
\label{sec::exp}
In this section, we evaluate our framework on two datasets, Rotated MNIST and CIFAR-10. All of the experiments are published on GitHub\footnote{\url{https://github.com/microljy/DREN_Tensorflow}}.

\subsection{Rotated MNIST}
First, we evaluate on Rotated MNIST dataset\cite{larochelle2007}. It is the rotated version of MNIST dataset\cite{mnist}, with digits being rotated for a degree uniformly drawn from $[0,2\pi]$. In total, it contains $62000$ handwritting digit images, among which there are $10000$ images for training, $2000$ for validating and $50000$ for testing. 

We choose the Z2CNN(refer to \cite{group}) as our baseline, which consists of $7$ convolutional layers of kernel size $3\times3$($4\times4$ in the last layer). Each layer of Z2CNN has $20$ channels, followed by a ReLU activation layer, a dropout layer, a batch normalization layer. Besides, a max pooling layer is applied right after the second layer.

We replace the first layer and the last layer of the baseline network by a cycle layer and a decycle layer, respectively. Besides, the intermediate layers are replaced by isotonic layers. Moreover, since that our proposed layers can reduce the number of parameters fourfold, we double the number of channels in each layers to keep the number of parameters approximately fixed.

Table \ref{tbl::rot-mnist} lists our results. This architecture(DREN) is found to achieve high accuracy with a small number of params. We also try to replace the decycle layer in DREN with a isotonic layer followed by a cross-channel max pooling layer. This network(DRENMaxPoolling) have 25k params., which is 24\% less than that of H-Net\cite{worrall2016harmonic}, and equal to that of P4CNN\cite{group} and Dieleman et al.\cite{dieleman2016exploiting}. DRENMaxPoolling outperforms all the previous models without more params.. In addition, our architecture can reduce about half of running time with less memory overhead compared with Dieleman et al.\cite{dieleman2016exploiting}, this would be further discussed in Section \ref{sec::time}. 





\begin{table}[!ht]
\caption{Performance of various models on Rotated MNIST.}
\label{tbl::rot-mnist}
\begin{center}
\begin{tabular}{|l|c|c|}
\hline
Model & Error & Param. \\
\hline
SVM\cite{larochelle2007}                     & 11.11\%           & - \\
Transformation RBM\cite{Sohn2012Learning}    & 4.2\%             & - \\
Conv-RBM\cite{schmidt2012learning}           & 3.98\%            & - \\ 
\hline
Z2CNN\cite{group}                            & 5.03\%            & 22k\\
P4CNN\cite{group}                            & 2.28\%            & 25k\\
H-Net\cite{worrall2016harmonic}              & 1.69\%            & 33k\\
Dieleman et al.\cite{dieleman2016exploiting} & 1.78\%            & 25k\\
DREN                                         & 1.78\%            & 22k\\
\textbf{DRENMaxPoolling}                    & \textbf{1.56\%}   & 25k\\
\hline
\end{tabular}
\end{center}
\end{table}

\subsection{CIFAR-10}
We also evaluate our framework on CIFAR-10\cite{krizhevsky2009learning}, a real-world dataset that does not have rotational symmetry in nature. It consists of $50000$ training and $10000$ testing images uniformly drawn from $10$ classes. Each one has RGB channels of size $32\times32$. For this dataset, we apply global contrast normalization to preprocess the data, as was used by \cite{goodfellow2013maxout}.

We test our framework on Network in Network(NIN)\cite{nin} and Resnet-20\cite{he2016deep}. First, we replace the ordinary convolutional layers of these models by our proposed layers, see r-NIN(conv1-4), r-Resnet-20(conv1-13), where conv1-k means that the first k layers are replaced by our proposed layers. Since our proposed layers only need a quarter of the number of parameters in each layer, we also evaluate r-NIN(conv1-4)$\times4$ and r-Resnet-20(conv1-13)$\times4$, where the number of channels of the first 4 layers are doubled to keep the total number of parameters unchanged. 
The results are shown in Table \ref{tbl::rot-cifar}. It shows that our model(r-NIN(conv1-4)$\times4$ and r-Resnet-20(conv1-13)$\times4$) can outperform the baseline with roughly the same number of parameters. In this experiment, we do not replace all of the convolution layers with our proposed layers, since we observe that simply replacing all layers cannot achieve the best performance. This would be discussed in the next subsection. 

\begin{table}[t]
\centering
\begin{threeparttable}
\caption{Performance of various models on CIFAR-10.}
\label{tbl::rot-cifar}
\begin{tabular}{|C{3.8cm}|C{1.5cm}|C{1.5cm}|}
\hline
Model                   & Error          & Params. \\
\hline
NIN\cite{nin}           & 10.41\%        & 967k\\
NIN*                    & 9.4\%          & 967k\\
r-NIN(conv1-4)          & 9.8\%          & \textbf{576k}\\
\textbf{r-NIN(conv1-4)$\times4$}& \textbf{9.0\%} & 958k\\
\hline
Resnet-20\cite{he2016deep}& 9.00\%          & 297k\\
r-Resnet-20(conv1-13)     & 8.51\%          & \textbf{245k}\\
\textbf{r-Resnet-20(conv1-13)$\times4$}     & \textbf{7.17}\% & 297k\\
\hline
\end{tabular}
\begin{tablenotes}
\footnotesize
\item [1] * means our implementation.
\item [2] conv1-k means 1-k layers are replaced by our proposed layers
\item [3] $r-$ means the model utilizes rotation equivariant layers.
\item [4] $\times4$ means the number of channels of proposed layers are doubled.
\end{tablenotes}
\end{threeparttable}
\end{table}

\pgfplotsset{compat=1.6,tick style={font=\small}}

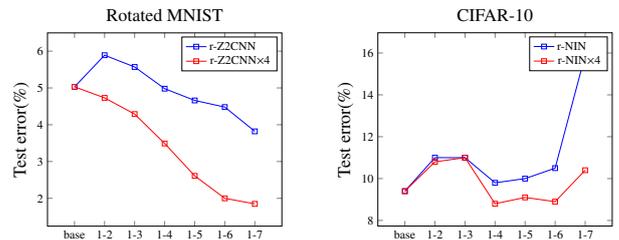
\begin{figure}
\label{fig::layers1}
\centering
\begin{minipage}{0.23\textwidth}
\centering
\scalebox{0.45}{
\begin{tikzpicture}
\begin{axis}[
    title=\scalebox{1.5}{Rotated MNIST},
    enlargelimits=0.15,
    legend style={anchor=north east,legend columns=1},
    ylabel={\scalebox{1.5}{Test error(\%)}},
    symbolic x coords={base, 1-2, 1-3, 1-4, 1-5, 1-6, 1-7},
    xtick=data,
    x tick label style={rotate=0},
    nodes near coords align={vertical},
    ]
\addplot[draw=blue,thick,mark=square,thick] 
    coordinates {(base,5.03) (1-2,5.89) (1-3,5.57)(1-4,4.98) (1-5,4.66)(1-6,4.48) (1-7,3.82) };
\addplot[draw=red,thick,mark=square,thick] 
    coordinates {(base,5.03) (1-2,4.73) (1-3,4.29)(1-4,3.49) (1-5,2.61)(1-6,2.00) (1-7,1.85) };
    \legend{r-Z2CNN$\quad~$,r-Z2CNN$\times4$}
\end{axis}
\end{tikzpicture}
}
\end{minipage}
\begin{minipage}{0.23\textwidth}
\centering
\scalebox{0.45}{
\begin{tikzpicture}
\begin{axis}[
    title=\scalebox{1.5}{CIFAR-10},
    enlargelimits=0.15,
    legend style={anchor=north east,legend columns=1},
    ylabel={\scalebox{1.5}{Test error(\%)}},
    symbolic x coords={base, 1-2, 1-3, 1-4, 1-5, 1-6, 1-7},
    xtick=data,
    x tick label style={rotate=0},
    nodes near coords align={vertical},
    ]
\addplot[draw=blue,thick,mark=square,thick] 
    coordinates {(base,9.4) (1-2,11.0) (1-3,11.0)(1-4,9.8) (1-5,10.0)(1-6,10.5) (1-7,15.9) };
\addplot[draw=red,thick,mark=square,thick] 
    coordinates {(base,9.4) (1-2,10.8) (1-3,11.0)(1-4,8.8) (1-5,9.1)(1-6,8.9) (1-7,10.4) };
    \legend{r-NIN$\quad~$,r-NIN$\times4$}
\end{axis}
\end{tikzpicture}
}
\end{minipage}
\caption{\textbf{Classification error of models with different numbers of our proposed layers}: The left one is the results on Rotated MNIST and the right one is that on CIFAR-10. The horizontal axis shows the number of our proposed layers applied. It turns out that, for datasets which has rotation symmetry in nature(Rotated MNIST), applying more isotonic layers yields better classification accuracy. We argue that higher level rotation equivariant representations learnt by stacking more isotonic layers improves the classification accuracy on such dataset. However, for CIFAR-10, the circumstance changes. Applying $4$ or $5$ isotonic layers achieves the best performance. We argue that higher level rotation equivariant representations are not helpful on such dataset without rotation symmetry, and applying more isotonic layers reduces the number of parameters(model complexity) fourfold, therefore leading to severe underfitting. $\times4$ means the model with doubled number of channels.}
\end{figure}

\begin{figure*}
\label{fig::visual1}
\begin{center}
\includegraphics[width=\textwidth]{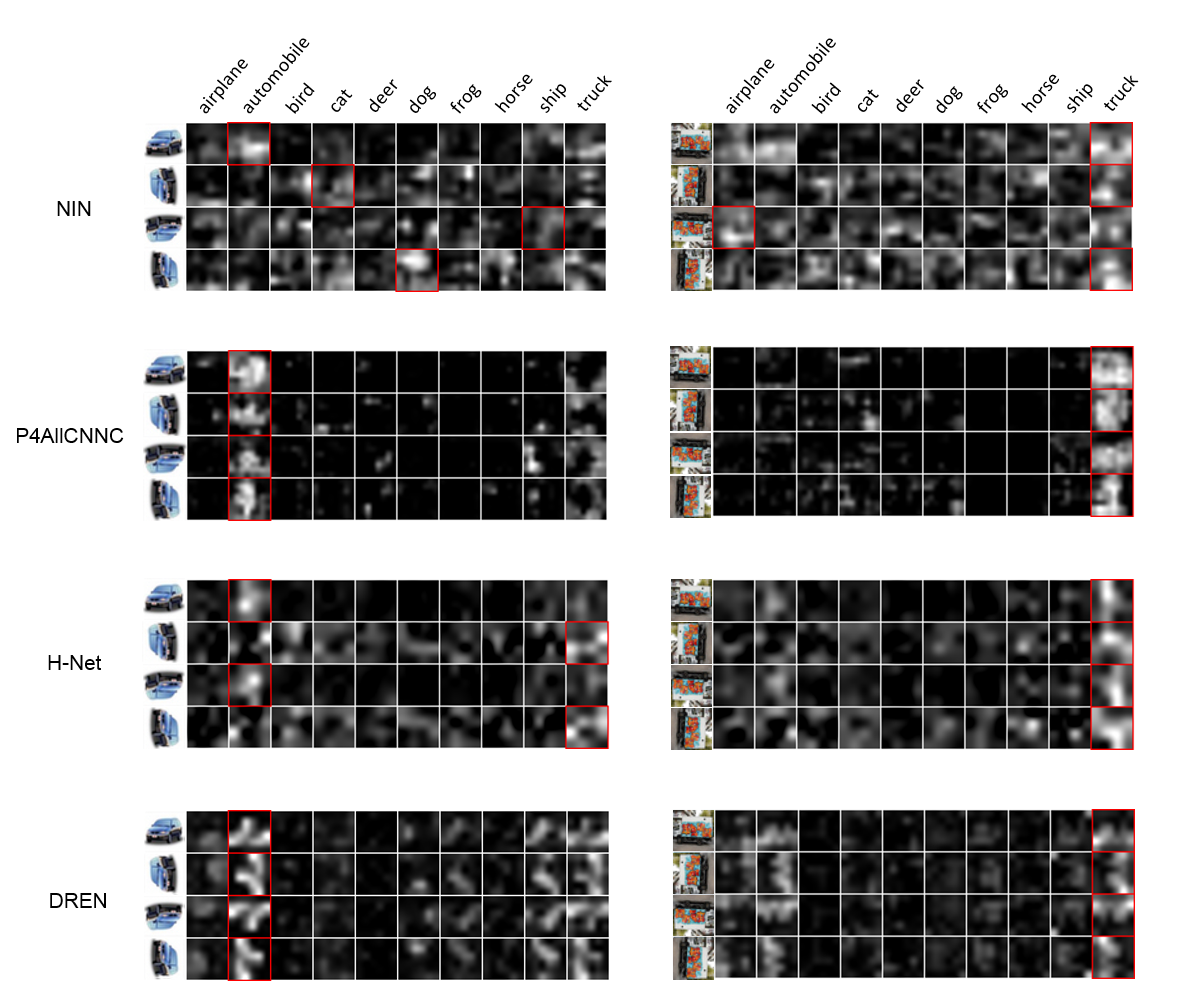}
\end{center}
\caption{\textbf{The representations in the last convolutional layer learnt by NIN\cite{nin}, H-Net\cite{worrall2016harmonic}, P4-ALL-CNN-C\cite{group} and DREN}. The two columns correspond to the representations of an image of automobile(left) and an image of truck(right). Inside one of the 8 pictures, the leftmost column contains the rotated input images while the remaining 10 columns contain the 10 feature maps in the last convolutional layers, which corresponds to the 10 classes of CIFAR-10. We slightly modify the architecture of H-Net to ensure that a global average pooling layer is used right after the last convolutional layer. The other three networks also apply this structure. Thus, the predictions of the networks are exactly the class whose related representation attains maximum density. We mark the maximal one(the prediction) with a red square. It turns out that when the input is rotated, NIN yields entirely different representations, leading to wrong predictions. For H-Net and P4-ALL-CNN-C, the representations are partly equivariant to rotation. However, H-Net still predicts incorrectly, when the image of automobile is rotated. In fact, only the representations of DREN are exactly equivariant to rotation, and thus the predictions of DREN are consistent, when the input is rotated. Moreover, we find our leanrt representations for automobile and for truck are very similar while that of other models are not. This shows that the representations learnt by DREN are similar for semantically similar objects, which indicates the reasonability of adding rotation equivariant constraint.}
\end{figure*}

\subsection{How many isotonic layers should be applied?}
The only hyperparameter of our framework is the number of isotonic layers used. Thus, in this subsection, we study the performance of DRENs with different numbers of isotonic layers in the following experiments. We also choose two baselines above, evaluate them on Rotated MNIST and CIFAR-10 respectively, and gradually integrate more isotonic layers into baselines(e.~g., conv1-2, conv1-3, conv1-4, etc.).

The results are shown in Figure 3. We observe two quite distinct trends in the two datasets. For Rotated MNIST which has rotation symmetry in nature, the more isotonic layers we use, the better performance the model achieves, even with less and less number of parameters. However, for CIFAR-10, r-NINs achieve best performance when using two isotonic layers(i.e., replace 4 convolutional layers). 

Because CIFAR-10 does not have deep rotation symmetry, intuitively, most effective latent representations maybe be not rotation equivariant. Thus, using too much isotonic layers prevents the network from learning such representations. This is one of the reason why our framework cannot achieve better performance with more isotonic layers. Another cause may be that the number of parameters(i.e., model complexity) decreases with more isotonic layers used, thus leading to a underfitting model.

The other extreme case that very a few isotonic layers are applied, also results in bad results. According to Figure 3, our models(r-Z2CNN,r-NIN,r-NIN$\times4$) underperform the baseline when less than two isotonic layers are applied. 
This observation actually meets with the argument we pose in Section \ref{sec::intro}, that the networks with a few or none isotonic layer can only learn shallow and trivial rotation equivariant representations.

\subsection{Running time of rotating filters and feature maps}
\label{sec::time}
In order to compare the running time costs between rotating feature maps and rotating filters, we evaluate both methods on Rotated MNIST dataset. We take the implementation presented by \cite{dieleman2016exploiting} as the representative of the methods that rotate feature maps. Since there are some difference between their framework and ours, we make minor modification on both frameworks to ensure that the comparison is a fair play. The decycle layer is replaced with a cross-channel mean-pooling layer in our framework, and the cyclic roll operation is applied before every convolution layer except the first one in their frameworks.

The running time of rotating filters and feature maps are shown in Table \ref{tbl::compare}. Both implementations run on Intel core i7-5930K processors(3.50GHz) and GeForce GTX 1080 with NVIDIA CUDA 8.0 and cuDNN 5. Since all models are implemented with Theano that automatically free memory using a reference count system, we cannot exactly determine memory consumption. Thus, only running time is listed. 

\begin{table}[!ht]
\caption{Running time of two implementations.} 
\label{tbl::compare}
\centering
\begin{threeparttable}
\begin{tabular}{|c|c|c|c|}

\hline
\multirow{2}{*}{Implementation} &
\multicolumn{2}{|c|}{Setting} &
\multirow{2}{*}{Time} \\
\cline{2-3}
 & Model & Batch size   &  \\
\hline
\multirow{4}{*}{Rotate filters}      & \multirow{2}{*}{Z2CNN} & 64  & 1.97s\\
                                     &                        & 128 & 1.44s\\
\cline{2-4}
                                     & \multirow{2}{*}{NIN}   & 64  & 11.00s\\
                                     &                        & 128 & 9.52s\\
\hline
\multirow{4}{*}{Rotate feature maps\cite{dieleman2016exploiting}} & \multirow{2}{*}{Z2CNN} & 64  & 4.15s\\
                                     &                        & 128 & 3.74s\\
\cline{2-4}
                                     & \multirow{2}{*}{NIN}   & 64  & 22.13s\\
                                     &                        & 128 & 18.73s\\
\hline
\end{tabular}
\begin{tablenotes}
\item [1] The time cost is the testing time of testing 50k images.
\end{tablenotes}
\end{threeparttable}
\vskip -0.1in
\end{table}

Table \ref{tbl::compare} shows that our implementation based on rotating filters are twice as fast as the implementation based on rotating feature maps\cite{dieleman2015rotation,laptev2016ti,dieleman2016exploiting}. In fact, the running time of the implementation based on rotating filters can be the same as the running time of the ordinary CNN, since the filters can be rotated and stored before running.

\subsection{Visualization of DREN}
Since deep rotation equivariant representations can be learnt from our framework, we are concerned with how such representations differ from ordinary representation of a typical CNN and other rotation equivariance structures. To this end, we visualized the representations produced by NIN\cite{nin}, P4-ALL-CNN-C\cite{group}, H-Net\cite{worrall2016harmonic} and DREN trained on CIFAR-10. We slight modify the architecture of H-Net to a global average pooling manner that are used by the other three models in order for fair comparison and clear visualization. In detail, we apply global average pooling layer rather than fully connected layer in above models. This indicates that the feature maps in the last convolutional layers are the confidence maps of the categories in fact. Such feature maps of four models are shown in Figure 4. 
In each picture, there are 11 columns. The leftmost one is the input image while the remaining ones are the feature maps in the last convolutional layer, each one corresponding to a class of CIFAR-10. In each row, we mark the predictions of the networks(i.e., the one with maximal density) with red squares.

It turns out that when the input is rotated, NIN yields entirely different representations, leading to wrong predictions. For H-Net and P4-ALL-CNN-C, the representations are partly equivariant to rotation. However, H-Net still predicts incorrectly, when the image of automobile is rotated. In fact, only the representations of DREN are exactly equivariant to rotation, and thus the predictions of DREN are consistent when the input is rotated. 

Moreover, we find our leanrt representations for automobile and for truck are very similar while that of other models are not. This shows that the representations learnt by DREN are similar for semantically similar objects, which indicates the reasonability of adding rotation equivariant constraint.

\section{Conclusion \& Discussion}
We propose a novel CNN framework, DREN, which can learn deep rotation equivariant representations from images. Theoretical guarantee is provided that the features are rotation equivariant to rotation of degree $k\pi/2$ for $k\in\mathbb{Z}$ and experiment results show that our framework can also deal with other degrees. We evaluate our framework on Rotated MNIST and CIFAR-10 datasets, outperforming several state-of-the-art CNN architectures. In addition, our implementation of rotating filters costs only a half time of feeding forward comparing to the methods which rotate feature maps. 

As a future direction, we suggest to settle the cases of rotation of general degree. Although our framework mainly deals with rotation of degree $k\pi/2$ for $k\in\mathbb{Z}$, it can be easily generalized to any finte group of rotation transformations, if there is an efficient way to perform rotations of any degree on filters. In short, our approach is a tradeoff between efficiency and performance.

For another direction, we would like to design new types of convolutional layers to implement equivariance with respect to scale transformations and more generally, affine transformations. These may improve the performance of CNNs further.

\section*{References}
\bibliography{reference}

\begin{thebibliography}{10}
\expandafter\ifx\csname url\endcsname\relax
  \def\url#1{\texttt{#1}}\fi
\expandafter\ifx\csname urlprefix\endcsname\relax\def\urlprefix{URL }\fi
\expandafter\ifx\csname href\endcsname\relax
  \def\href#1#2{#2} \def\path#1{#1}\fi

\bibitem{cnn1}
C.~Szegedy, W.~Liu, Y.~Jia, P.~Sermanet, S.~Reed, D.~Anguelov, D.~Erhan,
  V.~Vanhoucke, A.~Rabinovich, Going deeper with convolutions, in: Proceedings
  of the IEEE Conference on Computer Vision and Pattern Recognition, 2015, pp.
  1--9.

\bibitem{cnn2}
R.~Girshick, J.~Donahue, T.~Darrell, J.~Malik, Rich feature hierarchies for
  accurate object detection and semantic segmentation, in: Proceedings of the
  IEEE conference on computer vision and pattern recognition, 2014, pp.
  580--587.

\bibitem{cnn3}
A.~Krizhevsky, I.~Sutskever, G.~E. Hinton, Imagenet classification with deep
  convolutional neural networks, in: Advances in neural information processing
  systems, 2012, pp. 1097--1105.

\bibitem{cnn4}
K.~Simonyan, A.~Zisserman, Very deep convolutional networks for large-scale
  image recognition, arXiv preprint arXiv:1409.1556.

\bibitem{he2016deep}
K.~He, X.~Zhang, S.~Ren, J.~Sun, Deep residual learning for image recognition,
  in: Proceedings of the IEEE conference on computer vision and pattern
  recognition, 2016, pp. 770--778.

\bibitem{group}
T.~Cohen, M.~Welling, Group equivariant convolutional networks, in: Proceedings
  of the 33nd International Conference on Machine Learning, {ICML} 2016, New
  York City, NY, USA, June 19-24, 2016, 2016, pp. 2990--2999.

\bibitem{wu2015flip}
F.~Wu, P.~Hu, D.~Kong, Flip-rotate-pooling convolution and split dropout on
  convolution neural networks for image classification, arXiv preprint
  arXiv:1507.08754.

\bibitem{marcos2016learning}
D.~Marcos, M.~Volpi, D.~Tuia, Learning rotation invariant convolutional filters
  for texture classification, in: Pattern Recognition (ICPR), 2016 23rd
  International Conference on, IEEE, 2016, pp. 2012--2017.

\bibitem{teney2016learning}
D.~Teney, M.~Hebert, Learning to extract motion from videos in convolutional
  neural networks, arXiv preprint arXiv:1601.07532.

\bibitem{dieleman2016exploiting}
S.~Dieleman, J.~D. Fauw, K.~Kavukcuoglu, Exploiting cyclic symmetry in
  convolutional neural networks, in: Proceedings of the 33nd International
  Conference on Machine Learning, {ICML} 2016, New York City, NY, USA, June
  19-24, 2016, 2016, pp. 1889--1898.

\bibitem{kavukcuoglu2009learning}
K.~Kavukcuoglu, R.~Fergus, Y.~LeCun, et~al., Learning invariant features
  through topographic filter maps, in: Computer Vision and Pattern Recognition,
  2009. CVPR 2009. IEEE Conference on, IEEE, 2009, pp. 1605--1612.

\bibitem{norouzi2009stacks}
M.~Norouzi, M.~Ranjbar, G.~Mori, Stacks of convolutional restricted boltzmann
  machines for shift-invariant feature learning, in: Computer Vision and
  Pattern Recognition, 2009. CVPR 2009. IEEE Conference on, IEEE, 2009, pp.
  2735--2742.

\bibitem{schmidt2012learning}
U.~Schmidt, S.~Roth, Learning rotation-aware features: From invariant priors to
  equivariant descriptors, in: Computer Vision and Pattern Recognition (CVPR),
  2012 IEEE Conference on, IEEE, 2012, pp. 2050--2057.

\bibitem{lee2009convolutional}
H.~Lee, R.~Grosse, R.~Ranganath, A.~Y. Ng, Convolutional deep belief networks
  for scalable unsupervised learning of hierarchical representations, in:
  Proceedings of the 26th annual international conference on machine learning,
  ACM, 2009, pp. 609--616.

\bibitem{yu2017multi}
J.~Yu, C.~Hong, Y.~Rui, D.~Tao, Multi-task autoencoder model for recovering
  human poses, IEEE Transactions on Industrial Electronics PP~(99) (2017) 1--1.

\bibitem{yu2017iprivacy}
J.~Yu, B.~Zhang, Z.~Kuang, D.~Lin, J.~Fan, iprivacy: image privacy protection
  by identifying sensitive objects via deep multi-task learning, IEEE
  Transactions on Information Forensics and Security 12~(5) (2017) 1005--1016.

\bibitem{zhang2017learning}
J.~Zhang, K.~Li, Y.~Liang, N.~Li, Learning 3d faces from 2d images via stacked
  contractive autoencoder ☆, Neurocomputing 257 (2017) 67--78.

\bibitem{grzegorczyk2016encouraging}
K.~Grzegorczyk, M.~Kurdziel, P.~I. W{\'o}jcik, Encouraging orthogonality
  between weight vectors in pretrained deep neural networks, Neurocomputing 202
  (2016) 84--90.

\bibitem{fasel2006rotation}
B.~Fasel, D.~Gatica-Perez, Rotation-invariant neoperceptron, in: 18th
  International Conference on Pattern Recognition (ICPR'06), Vol.~3, IEEE,
  2006, pp. 336--339.

\bibitem{dieleman2015rotation}
S.~Dieleman, K.~W. Willett, J.~Dambre, Rotation-invariant convolutional neural
  networks for galaxy morphology prediction, Monthly notices of the royal
  astronomical society 450~(2) (2015) 1441--1459.

\bibitem{gens2014deep}
R.~Gens, P.~M. Domingos, Deep symmetry networks, in: Advances in neural
  information processing systems, 2014, pp. 2537--2545.

\bibitem{worrall2016harmonic}
D.~E. Worrall, S.~J. Garbin, D.~Turmukhambetov, G.~J. Brostow, Harmonic
  networks: Deep translation and rotation equivariance, arXiv preprint
  arXiv:1612.04642.

\bibitem{marcos2016rotation}
D.~M. Gonzalez, M.~Volpi, D.~Tuia, Learning rotation invariant convolutional
  filters for texture classification, CoRR abs/1604.06720.

\bibitem{nin}
M.~Lin, Q.~Chen, S.~Yan, Network in network, in: In Proc. ICLR, 2014.

\bibitem{laptev2016ti}
D.~Laptev, N.~Savinov, J.~M. Buhmann, M.~Pollefeys, Ti-pooling:
  transformation-invariant pooling for feature learning in convolutional neural
  networks, arXiv preprint arXiv:1604.06318.

\bibitem{gemm}
V.~Volkov, J.~W. Demmel, Benchmarking gpus to tune dense linear algebra, in:
  High Performance Computing, Networking, Storage and Analysis, 2008. SC 2008.
  International Conference for, IEEE, 2008, pp. 1--11.

\bibitem{larochelle2007}
H.~Larochelle, D.~Erhan, A.~Courville, J.~Bergstra, Y.~Bengio, An empirical
  evaluation of deep architectures on problems with many factors of variation,
  in: Proceedings of the 24th international conference on Machine learning,
  ACM, 2007, pp. 473--480.

\bibitem{mnist}
Y.~LeCun, L.~Bottou, Y.~Bengio, P.~Haffner, Gradient-based learning applied to
  document recognition, Proceedings of the IEEE 86~(11) (1998) 2278--2324.

\bibitem{Sohn2012Learning}
K.~Sohn, H.~Lee, Learning invariant representations with local transformations,
  arXiv preprint arXiv:1206.6418.

\bibitem{krizhevsky2009learning}
A.~Krizhevsky, G.~Hinton, Learning multiple layers of features from tiny
  images.

\bibitem{goodfellow2013maxout}
I.~Goodfellow, D.~Warde-Farley, M.~Mirza, A.~Courville, Y.~Bengio, Maxout
  networks, in: Proceedings of The 30th International Conference on Machine
  Learning, 2013, pp. 1319--1327.

\end{thebibliography}

\end{document}